\documentclass{article}
\pdfoutput=1 
\usepackage{arxiv}
\renewcommand{\undertitle}{Under consideration for publication in Information Sciences (Inf. Sci.)}
\renewcommand{\headeright}{Preprint}
\renewcommand{\shorttitle}{On the Triangle Inequality for the Jaccard Distance in Arbitrary Lattices} 

\usepackage{amssymb}
\usepackage{amsmath}
\usepackage{amsthm}

\newtheorem{definition}{Definition} 
\newtheorem{proposition}{Proposition} 
\newtheorem{theorem}{Theorem} 
\title{On the Triangle Inequality for the Jaccard Distance in Arbitrary Lattices}
\author{
  Costin B\u{a}dic\u{a} \\
  Department of Computers and Information Technology \\
  University of Craiova \\
  Craiova, Romania \\
  \texttt{costin.badica@edu.ucv.ro} 
  \And
  Amelia B\u{a}dic\u{a} \\
  Department of Business Informatics \\
  University of Craiova \\
  Craiova, Romania \\
  \texttt{amelia.badica@edu.ucv.ro}
}
\date{August 2026}

\begin{document}

\maketitle

\begin{abstract}
This paper presents new theoretical results on generalizing the Jaccard distance for lattices and real valuations. We demonstrate that when the valuation is strictly positive, monotone, and modular, the Jaccard distance satisfies the triangle inequality on arbitrary lattices, effectively generalizing earlier results that depended heavily on distributivity. Moving to relatively complemented distributive lattices (which safely drop the requirement for the global bounds found in Boolean algebras), we prove the triangle inequality holds as long as the valuation is positive, monotone, supermodular, and $\log$-submodular. Additionally, we adapt the symmetric-difference Jaccard formulation for submodular valuations to sectionally complemented distributive lattices. Shifting to necessary conditions, we prove that supermodularity is a strict requirement for the standard generalized Jaccard distance to operate as a valid metric. Finally, we map the practical value of relaxing these structural constraints to computational fields like quantum information theory, formal concept analysis, and machine learning, closing with a brief look at open mathematical problems.
\end{abstract}


\noindent\textbf{Keywords:} Generalized Jaccard distance; Arbitrary lattices; Triangle inequality; Valuation function; Modularity; Supermodularity; Submodularity


\section{Introduction}
\label{sec:intro}

Paul Jaccard originally introduced his similarity index in 1901 to quantify floral biodiversity \cite{Jaccard1901}. Today, it remains a fundamental tool across computer science and mathematics for evaluating the similarity of two finite sets as a real number in $[0,1]$. 
Given two finite sets $X$ and $Y$, their Jaccard similarity is formulated as:
\begin{equation}
\label{eq:jaccard-sim}
\begin{aligned}
J(X,Y) = \frac{|X \cap Y|}{|X \cup Y|}
\end{aligned} 
\end{equation}

Also known as the Tanimoto similarity \cite{Lipkus1999} in chemical information systems and as the Intersection over Union (IoU) metric \cite{Rezatofighi2019} in computer vision, this measure calculates the proportion of shared elements against the total unique elements across both sets. It is strictly normalized: it yields $0$ if the sets are entirely disjoint and $1$ if they are identical. 

Whenever a similarity measure $S$ is normalized to $[0,1]$, we can naturally define a corresponding distance measure as $d_S(X,Y) = 1 - S(X,Y)$. A well-established feature of the distance associated with the Jaccard similarity, $d_J$, is that it forms a true metric, successfully satisfying the triangle inequality \cite{Levandowsky1971}.

Any rigorous distance metric must satisfy the triangle inequality. Mathematically, it ensures that a direct path between two points is never longer than a path taken through an intermediate point. In practical computer science, this property is an absolute requirement for algorithmic correctness and efficiency. For example, the triangle inequality provides the strict bounding rules necessary to prune branches in exact nearest-neighbor similarity searches \cite{Schubert2021}, safely skip redundant distance calculations to accelerate $k$-means clustering \cite{Elkan2003}, and guarantee valid mathematical approximation bounds for NP-hard routing models like the metric traveling salesperson problem \cite{Karlin2021}.

Over the decades, researchers have verified the metricity of the classic Jaccard distance using various strategies. Levandowsky and Winter \cite{Levandowsky1971} provided an early proof using a seven-region decomposition on a Venn diagram for $A \cup B \cup C$. Shortly after, Gilbert \cite{Gilbert1972} simplified the process with a different algebraic grouping of the same decomposition. Later proofs introduced scalar products of bit vectors \cite{Lipkus1999} and arguments by contradiction \cite{Grygorian2018}. However, a recurring limitation across these classic proofs is their strict dependence on Boolean algebra axioms, specifically distributivity and the existence of set complements. These specific properties are simply not available in arbitrary lattices.

As the Jaccard index was adapted for more complex data over time, these structural limits prompted researchers to create numerous domain-specific variations. For example, the weighted Jaccard similarity handles real-valued vectors \cite{Chierichetti2010}, while various fuzzy metrics use $t$-norms and $t$-conorms to evaluate fuzzy sets \cite{DeBaets1999}. Similar extensions, alongside broader studies of normalized similarity measures, have been thoroughly developed for Boolean models \cite{Badica2025}, multisets, graphs, and probability spaces. While these variants work perfectly well within their specific niches, their metricity proofs typically lean entirely on the unique algebraic quirks of those exact data structures. 

Instead of adding another domain-specific formula to the pile, recent theoretical work focuses on unifying these applications. This is done by abstracting the data structures into lattice theory and swapping out specific size operators (like cardinality or vector norms) for an arbitrary real-valued valuation function $f$. Kosub \cite{Kosub2019} formalized this approach by defining the generalized Jaccard similarity as:
\begin{equation}
\label{eq:gen-jaccard-sim}
\begin{aligned}
J_f(X,Y) = \frac{f(X \cap Y)}{f(X \cup Y)}
\end{aligned} 
\end{equation}

When $f$ is positive and monotone, $J_f$ stays safely normalized in $[0,1]$. Kosub showed (Theorem 3 in \cite{Kosub2019}) that if $f$ is positive, monotone, and modular, the resulting generalized distance $d_{f,J}$ satisfies the triangle inequality. Although originally framed for sets, the proof easily scales to Boolean algebras. Yet, because it still relies explicitly on the distributivity axiom, it does not seamlessly transfer to arbitrary lattices.

Other research tracks have tried relaxing the modularity requirement of the valuation function itself. For example, generalized Consonni and Todeschini measures emerged by softening modularity into supermodularity and $\log$-submodularity, a topic that remains under active investigation for Boolean models \cite{Badica2025}. Kosub \cite{Kosub2019} also looked into submodular valuations using the symmetric-difference operator, which, again, is not natively defined in arbitrary lattices. Furthermore, recent studies have focused on axiomatic characterizations of the Jaccard metric, exploring quasi-metrics that emerge when strict constraints like the triangle inequality are relaxed \cite{Gerasimou2024}. 

Ultimately, most literature covering the generalized Jaccard distance leans heavily on the strict rules of Boolean algebras, particularly distributivity and global lattice bounds. Naturally, these constraints bottleneck the distance's usefulness in modern environments that lack these structures, such as quantum logic or open-ended data streams.

This paper systematically removes these structural roadblocks. By shifting our focus away from the constraints on the lattice and toward the inherent properties of the valuation function, we chart the exact mathematical boundaries where the generalized Jaccard distance remains a valid metric. Our primary contributions are:
\begin{itemize}
    \item[i)] \textbf{Eliminating the distributivity requirement:} We prove the generalized Jaccard distance holds up to the triangle inequality for modular valuations on arbitrary lattices. This shows that Gilbert's classic proof actually hinges on the monotony and modularity of the valuation function, not on Boolean distributivity. 
    \item[ii)] \textbf{Eliminating global lattice bounds:} We verify that a universal top ($\top$) or bottom ($\bot$) element is unnecessary to guarantee metricity for supermodular, $\log$-submodular, and submodular valuations. Using relatively and sectionally complemented distributive lattices, we show the metric works perfectly well within localized intervals.
    \item[iii)] \textbf{Establishing necessary conditions:} While older research focused almost entirely on sufficient conditions for metricity, we outline the necessary conditions. We prove that the supermodularity of $f$ and $\frac{1}{f}$ is a hard mathematical requirement for the standard generalized formulation to act as a metric, creating a simple theoretical filter for testing new valuation functions.
    \item[iv)] \textbf{Mapping applications:} We tie these structural generalizations back to real computational tasks, showing how the distance metric theoretically applies to non-distributive orthomodular lattices (in quantum information theory) and unbounded dynamic data structures (in machine learning on continuous streams).
\end{itemize} 

These results span all scenarios related to the modularity, supermodularity, or submodularity of $f$. Moreover, they highlight the deep connections between the implicit assumptions found in classic proofs and the metric demands placed on valuation functions.

The rest of the paper unfolds as follows: Section \ref{sec:background} covers the essential theoretical background. Section \ref{sec:modular} details the triangle inequality proof for modular valuations on arbitrary lattices. Section \ref{sec:super-sub-modular} pushes this analysis into supermodular and $\log$-submodular valuations on relatively complemented distributive lattices. Section \ref{sec:sub} tweaks previous results regarding submodular valuations to fit sectionally complemented distributive lattices. Section \ref{sec:nec} outlines the necessary conditions for the triangle inequality. Section \ref{sec:applications} looks at the practical value across information theory and machine learning. Section \ref{sec:open-problems} maps out open theoretical problems, and Section \ref{sec:concl} concludes the work.

\section{Background}
\label{sec:background}

Following Birkhoff \cite{Birkhoff1973}, we can define lattices as either partially ordered sets or algebraic structures. This section briefly reviews the basic definitions, the properties of specific lattice types, and real valuations on lattices.

\begin{definition} (Lattice: order definition)
\begin{itemize}
\item[i)] A partially ordered set (poset) is a set equipped with a binary relation $\leq$ that is reflexive, antisymmetric, and transitive.
\item[ii)] A lattice is a poset where any two elements $X$ and $Y$ share a greatest lower bound (the ``meet'', denoted $X \wedge Y$) and a least upper bound (the ``join'', denoted $X \vee Y$).
\end{itemize}
\end{definition}

\begin{definition} (Lattice: algebraic definition)
Any system utilizing two operations $\vee$ and $\wedge$ that are idempotent, commutative, associative, and follow the absorption laws qualifies as a lattice. 
\end{definition}

The connection lemma \cite{Birkhoff1973} confirms these two definitions are perfectly equivalent, defining the ordering relation as:
\begin{equation}
\label{eq:lattice-order}
X \leq Y \iff X \vee Y = Y \text{ and } X \wedge Y = X
\end{equation}

We can evaluate the elements of a lattice $L$ using a real valuation function $f : L \rightarrow \mathbb{R}$, which allows us to measure similarities between those elements.

\begin{definition} (Classes of valuation functions) Let $f : L \rightarrow \mathbb{R}$ act as a valuation function for lattice $L$.
\begin{itemize}
    \item[i)] $f$ is positive if and only if:
    \begin{equation}
    \label{eq:positive}
    f(X) \geq 0 \quad \forall X\in L
    \end{equation}
    If $f(X) > 0$ for all $X \in L$, then $f$ is strictly positive.
    \item[ii)] $f$ is modular if and only if:
    \begin{equation}
    \label{eq:modular}
    f(X) + f(Y) = f(X \vee Y) + f(X \wedge Y) \quad \forall X,Y \in L
    \end{equation}
    \item[iii)] $f$ is submodular if and only if:
    \begin{equation}
    \label{eq:submodular}
    f(X) + f(Y) \geq f(X \vee Y) + f(X \wedge Y) \quad \forall X,Y \in L
    \end{equation}
    \item[iv)] If $f$ is strictly positive and $\log \circ f$ is submodular, then $f$ is $\log$-submodular.
    \item[v)] $f$ is supermodular if and only if:
    \begin{equation}
    \label{eq:supermodular}
    f(X) + f(Y) \leq f(X \vee Y) + f(X \wedge Y) \quad \forall X,Y \in L
    \end{equation}
    \item[vi)] $f$ is monotone if and only if:
    \begin{equation}
    \label{eq:monotone}
    f(X) \leq f(Y) \quad \forall X,Y \in L \text{ such that } X \leq Y
    \end{equation}
    \item[vii)] $f$ satisfies unconditioned additive diminishing returns if and only if:
    \begin{equation}
    \label{eq:uncon-add-dim-ret}
    f(X\vee Z) - f(X) \geq f(Y\vee Z) - f(Y) \quad \forall X,Y,Z \in L \text{ such that } X \leq Y
    \end{equation}
    \item[viii)] A strictly positive $f$ satisfies unconditioned multiplicative diminishing returns if and only if:
    \begin{equation}
    \label{eq:uncon-mul-dim-ret}
    \frac{f(X\vee Z)}{f(X)} \geq \frac{f(Y\vee Z)}{f(Y)} \quad \forall X,Y,Z \in L \text{ such that } X \leq Y
    \end{equation}
\end{itemize}
\end{definition}

Adding extra axioms creates various specific types of lattices by further constraining the available operations \cite{Roman2008}. While a significant portion of our analysis applies to arbitrary lattices, we also frequently utilize distributive, relatively complemented, and lower bounded lattices, defined as follows:

\begin{definition} (Types of lattices)
\begin{itemize}
    \item[i)] A distributive lattice satisfies the distributivity of $\vee$ with $\wedge$, and $\wedge$ with $\vee$:
    \begin{equation}
    \label{eq:distrib}
    \begin{aligned}
    (X \vee Y) \wedge Z &= (X \wedge Z) \vee (Y \wedge Z) \quad \forall X,Y,Z \in L \\
    (X \wedge Y) \vee Z &= (X \vee Z) \wedge (Y \vee Z) \quad \forall X,Y,Z \in L
    \end{aligned} 
    \end{equation}
    \item[ii)] A lattice is lower bounded if it contains a unique bottom element $\bot \in L$ where $\bot \leq X$ for all $X \in L$. 
    \item[iii)] A lattice $L$ is relatively complemented if, for every $U, V \in L$ where $U \leq V$, and for every $A \in [U,V]$, there is a complement $X \in [U,V]$ of $A$ making $X \wedge A = U$ and $X \vee A = V$. In simpler terms, every closed interval sublattice $[U,V]$ acts as a complemented lattice \cite{Roman2008}.
    \item[iv)] A lattice $L$ is sectionally complemented if it is lower bounded and, for every $V \in L$, the closed interval sublattice $[\bot,V]$ is a complemented lattice \cite{Roman2008}.
\end{itemize}
\end{definition}

In any distributive complemented lattice (which includes relatively or sectionally complemented versions), an element's complement is unique. This rule flows directly from Theorem 10 in \cite{Birkhoff1973} (p. 10).

Suppose $X$ and $Y$ belong to a sectionally complemented distributive lattice with $X \leq Y$ (meaning $X \in [\bot,Y]$). We denote the unique relative complement of $X$ regarding $Y$ (its unique complement within the sublattice $[\bot,Y]$) as $X_Y^\prime$.

\section{Result for Modular Valuations}
\label{sec:modular}

Previous studies show that when the valuation function $f$ is modular, the generalized Jaccard distance cleanly satisfies the triangle inequality on Boolean and distributive lattices \cite{Kosub2019}. Here, we expand this finding to arbitrary (non-distributive) lattices. 

To do this, we adapt Gilbert's concise proof \cite{Gilbert1972}, originally designed for the powerset lattice. By isolating Gilbert's logical needs into a single generalized condition, we show that the distributivity axioms are not actually required. 

\begin{proposition}\label{prop:gilbert-cond-mon-mod} (The Gilbert condition equals modularity and monotony)
A valuation function $f$ on an arbitrary lattice $L$ satisfies the Gilbert condition:
\begin{equation}
\label{eq:gilbert-cond}
\begin{split}
f(A \vee C) + f(C \vee B) + f(A \wedge B) \geq {} & f(A\wedge C) + f(C\wedge B) \\
& + f(A\vee B \vee C) \quad \forall A,B,C \in L
\end{split}
\end{equation}
if and only if $f$ is modular and monotone.
\end{proposition}

\begin{proof}
$(\Rightarrow)$ First, assume the Gilbert condition is true.

Setting $A=X$, $B=Y$, and $C=X\wedge Y$ in (\ref{eq:gilbert-cond}) produces submodularity:
$$f(X) + f(Y) + f(X \wedge Y) \geq f(X \wedge Y) + f(X \wedge Y) + f(X \vee Y)$$
$$f(X) + f(Y) \geq f(X \vee Y) + f(X \wedge Y)$$

Next, setting $A=X$, $B=Y$, and $C=X\vee Y$ in (\ref{eq:gilbert-cond}) reveals supermodularity:
$$f(X \vee Y) + f(X \vee Y) + f(X \wedge Y) \geq f(X) + f(Y) + f(X \vee Y)$$
$$f(X \vee Y) + f(X \wedge Y) \geq f(X) + f(Y)$$

Putting these two inequalities together proves that $f$ is strictly modular.

Finally, substituting $A=X$, $B=X$, and $C=Y$ (where $X \geq Y$) into (\ref{eq:gilbert-cond}) gives:
$$f(X) + f(X) + f(X) \geq f(Y) + f(Y) + f(X)$$
$$2f(X) \geq 2f(Y) \implies f(X) \geq f(Y)$$
confirming that $f$ is monotone.

$(\Leftarrow)$ Conversely, assume $f$ is modular and monotone. Applying the modularity rule to $A \vee C$ and $B$, we get:
\begin{equation}
\label{eq:from-modularity}
f(A \vee B \vee C) = f(A \vee C) + f(B) - f((A\vee C) \wedge B)
\end{equation}

Plugging (\ref{eq:from-modularity}) into the target Gilbert condition (\ref{eq:gilbert-cond}) and shuffling the terms, we need to prove:
\begin{equation}
\label{eq:to-prove-1}
f(C \vee B) + f(A \wedge B) + f((A\vee C) \wedge B) \geq f(A\wedge C) + f(C\wedge B) + f(B)
\end{equation}

Using modularity on $X=A\wedge C$ and $Y=B$ results in $f(A\wedge C) + f(B) = f((A\wedge C) \vee B) + f(A \wedge B \wedge C)$. Since $f$ is monotone and $A \wedge B \geq A \wedge B \wedge C$, it naturally follows that $f(A \wedge B) \geq f(A \wedge B \wedge C)$. This creates a helpful lower bound for $f(A \wedge B)$:
\begin{equation}
\label{eq:lower-bound}
f(A \wedge B) \geq f(A \wedge C) + f(B) - f((A\wedge C) \vee B)
\end{equation}

Inserting (\ref{eq:lower-bound}) back into (\ref{eq:to-prove-1}) and canceling terms simplifies our target to:
\begin{equation}
\label{eq:to-prove-2}
f(C \vee B) + f((A\vee C)\wedge B) \geq f((A\wedge C) \vee B) + f(C\wedge B)
\end{equation}

Because $C \geq A \wedge C$, it is clearly true that $C\vee B \geq (A\wedge C) \vee B$. Relying on $f$'s monotonicity, we get $f(C \vee B) \geq f((A\wedge C)\vee B)$. 
Likewise, since $A\vee C \geq C$, we know $(A\vee C) \wedge B \geq C \wedge B$. Again by monotonicity, $f((A\vee C) \wedge B) \geq f(C\wedge B)$. 
Adding these two inequalities together lands exactly on (\ref{eq:to-prove-2}), wrapping up the proof.
\end{proof}

\begin{proposition}\label{prop-gilbert-triangle} (The Gilbert condition guarantees the triangle inequality)
If the Gilbert condition (\ref{eq:gilbert-cond}) holds for a strictly positive valuation $f$, then the generalized distance $d_{f,J}(\cdot,\cdot) = 1 - J_f(\cdot,\cdot)$ on arbitrary lattices safely satisfies the triangle inequality.
\end{proposition}

\begin{proof}
Recall the generalized Jaccard distance definition:
\begin{equation}
\label{eq:gen-jaccard-dist}
d_{f,J}(A,B) = 1 - \frac{f(A \wedge B)}{f(A \vee B)} = \frac{f(A \vee B) - f(A \wedge B)}{f(A \vee B)}
\end{equation}

Let $T = A \vee B \vee C$. Based on Proposition \ref{prop:gilbert-cond-mon-mod}, $f$ is monotone, which means $f(A \vee B) \leq f(T)$. Given that $f$ is strictly positive, we can set an upper bound for the distance:
\begin{equation}
\label{eq:upper-bound}
d_{f,J}(A,B) \leq 1 - \frac{f(A \wedge B)}{f(T)} = \frac{f(T) - f(A \wedge B)}{f(T)}
\end{equation}

Because the numerator $f(A \vee B) - f(A \wedge B)$ is non-negative (thanks to monotonicity) and $f(A \vee B) \leq f(T)$, increasing the denominator produces a solid lower bound:
\begin{equation}
\label{eq:lower-bound-dist}
d_{f,J}(A,B) = \frac{f(A \vee B) - f(A \wedge B)}{f(A \vee B)} \geq \frac{f(A \vee B) - f(A \wedge B)}{f(T)}
\end{equation}

Adding the inequalities from (\ref{eq:lower-bound-dist}) for $d_{f,J}(A,C)$ and $d_{f,J}(C,B)$ gives us:
\begin{equation}
\label{eq:to-prove-3}
d_{f,J}(A,C) + d_{f,J}(C,B) \geq \frac{f(A \vee C) - f(A \wedge C) + f(C \vee B) - f(C \wedge B)}{f(T)}
\end{equation}

To confirm the triangle inequality $d_{f,J}(A,C) + d_{f,J}(C,B) \geq d_{f,J}(A,B)$, we compare (\ref{eq:to-prove-3}) and (\ref{eq:upper-bound}). We just need to show:
\begin{equation}
\label{eq:to-prove-4}
\frac{f(A \vee C) - f(A\wedge C) + f(C\vee B) - f(C \wedge B)}{f(T)} \geq \frac{f(T) - f(A\wedge B)}{f(T)}
\end{equation}

Since $f(T) > 0$, we can safely drop the denominators. Rearranging the remaining terms brings us directly back to the Gilbert condition (\ref{eq:gilbert-cond}), which we assumed is true.
\end{proof}

This leads directly to the main result of this section, combining Propositions \ref{prop:gilbert-cond-mon-mod} and \ref{prop-gilbert-triangle}.

\begin{theorem}\label{th:mod}
If a valuation function $f$ on an arbitrary lattice $L$ is strictly positive, monotone, and modular, then the resulting distance $d_{f,J}$ satisfies the triangle inequality.
\end{theorem}

Note that while Theorem \ref{th:mod} guarantees the triangle inequality, the resulting distance $d_{f,J}$ technically forms a pseudometric unless $f$ is strictly monotone, which is required to satisfy the identity of indiscernibles ($d_{f,J}(A,B) = 0 \iff A = B$).

Proposition \ref{prop-gilbert-triangle} demonstrates how Gilbert's condition (\ref{eq:gilbert-cond}) distills the minimum mathematical requirements supporting the triangle inequality. By proving that this condition is mathematically equivalent to the monotony and modularity of positive valuations, we establish that the full suite of Boolean axioms is unnecessary.

Looking closely at Kosub's original proof (Theorem 3 in \cite{Kosub2019}) adds a bit more context. Kosub relied on distributivity specifically to prove $(A \vee C) \wedge (B \vee C) = (A \wedge B) \vee C$. However, strict equality isn't actually required to finish the proof. Showing the inequality is enough:
\begin{equation}
\label{eq:ineq-th3}
(A \vee C) \wedge (B \vee C) \geq (A \wedge B) \vee C
\end{equation}

In any arbitrary lattice, $A \vee C \geq C$ and $B \vee C \geq C$, pushing $(A \vee C) \wedge (B \vee C) \geq C$. By the same logic, $A \vee C \geq A$ and $B \vee C \geq B$, meaning $(A \vee C) \wedge (B \vee C) \geq A \wedge B$. Combining these bounds comfortably proves (\ref{eq:ineq-th3}). With this small adjustment, Kosub's proof also holds up for arbitrary lattices without needing the distributivity axiom.

Finally, as Rutherford noted \cite{Rutherford1965} (p. 20), any lattice supporting a monotone and modular valuation is categorized as a \textit{metric lattice}, making it inherently modular. A straightforward distance measure can be set on these structures as $\delta(X,Y) = f(X\vee Y) - f(X \wedge Y)$. Theorem \ref{th:mod} simply offers $d_{f,J}$ as a normalized, practical alternative for these metric lattices.

Alternatively, by taking the distance $\delta$ and applying the biotope transform \cite{Deza2016} (p. 88), one can find a different path to prove metricity for $d_{f,J}$ under tighter constraints, specifically, when the lattice has a lower bound $\bot$ and $f(\bot) = 0$. The biotope transform produces this metric distance for lower-bounded lattices:
\begin{equation}
\label{eq:gen-jaccard-distance-lower-bound}
d_{f,J}^{\bot}(A,B) = 1 - \frac{f(A \wedge B) - f(\bot)}{f(A \vee B) - f(\bot)}
\end{equation}

Notice that if a lattice has a lower bound $\bot$, $f$ is monotone, and $f(\bot) = 0$, the function $f$ naturally becomes positive. This means explicit positivity is not a separate requirement in this highly specific edge case.

\section{Result for Supermodular and $\log$-Submodular Valuations}
\label{sec:super-sub-modular}

Here, we ease up on the strict modularity requirement. We prove that the triangle inequality holds up fine if the valuation function is supermodular and $\log$-submodular, as long as the underlying lattice features relative complementarity and distributivity.

\begin{proposition}\label{prop:submodular} (Properties of submodular and $\log$-submodular valuations)
For any valuation function $f$ on an arbitrary lattice $L$:
\begin{itemize}
\item[i)] If $f$ is monotone and submodular, it naturally satisfies unconditioned additive diminishing returns (\ref{eq:uncon-add-dim-ret}).
\item[ii)] If $f$ is strictly positive, monotone, and $\log$-submodular, it satisfies unconditioned multiplicative diminishing returns (\ref{eq:uncon-mul-dim-ret}).
\end{itemize}
\end{proposition}

\begin{proof}
i) Assuming $f$ is monotone and submodular, let $X \leq Y$. Using the submodularity definition (\ref{eq:submodular}) on the elements $X\vee Z$ and $Y$ yields:
\begin{equation*}
f(X\vee Z) + f(Y) \geq f(X\vee Z \vee Y) + f((X \vee Z) \wedge Y) 
\end{equation*}

Because $X \leq Y$, we know $X \vee Y = Y$, which neatly simplifies the first term on the right to $f(Y \vee Z)$. 

Also, since $X \leq X \vee Z$ and $X \leq Y$, it follows that $X \leq (X \vee Z) \wedge Y$. Thanks to the monotonicity of $f$:
\begin{equation*}
f((X \vee Z) \wedge Y) \geq f(X)
\end{equation*}

Dropping these bounds back into our initial inequality gives:
\begin{equation*}
f(X \vee Z) + f(Y) \geq f(Y \vee Z) + f(X)
\end{equation*}
Shuffling the terms directly maps to the unconditioned additive diminishing returns condition (\ref{eq:uncon-add-dim-ret}).

ii) If $f$ is strictly positive and $\log$-submodular, the composition $\log\circ f$ is automatically submodular and monotone. Passing this through the result from point (i) gives:
\begin{equation*}
\log(f(X\vee Z)) - \log(f(X)) \geq \log(f(Y \vee Z)) - \log(f(Y)) 
\end{equation*}
By standard logarithm rules, exponentiating both sides instantly confirms the unconditioned multiplicative diminishing returns condition (\ref{eq:uncon-mul-dim-ret}).
\end{proof}

\begin{theorem}\label{th:super-log-sub}
If a valuation function $f$ on a relatively complemented distributive lattice $L$ is strictly positive, monotone, supermodular, and $\log$-submodular, then the resulting distance $d_{f,J}$ defined by (\ref{eq:gen-jaccard-dist}) satisfies the triangle inequality.
\end{theorem}

\begin{proof}
Proving the triangle inequality essentially requires proving this algebraic equivalent:
\begin{equation}
\label{eq:triangle_ineq_equiv}
1+ \frac{f(A \wedge C)}{f(A \vee C)} \geq \frac{f(A \wedge B)}{f(A \vee B)} + \frac{f(B \wedge C)}{f(B \vee C)} \quad \forall A,B,C \in L
\end{equation}

Let $V = A \wedge B \wedge C$ and $T = A \vee B \vee C$. Since $L$ is a relatively complemented distributive lattice, the localized closed interval sublattice $[V, T]$ behaves just like a Boolean algebra. 

Let $C^\prime$ be the unique relative complement of $A\vee B$ inside this interval $[V, T]$. By its definition:
\begin{equation}
\label{eq:complement_a_or_b}
(A\vee B) \wedge C^\prime = V \quad \text{and} \quad (A \vee B) \vee C^\prime = T
\end{equation} 

We first need to confirm that $C^\prime \leq C$. Because $C^\prime$ lives in $[V, T]$, we know $C^\prime = C^\prime \wedge T$. If we expand $T$ and apply distributivity:
\begin{equation}
\label{eq:cprime_upper_bound_proof}
C^\prime = C^\prime \wedge (A \vee B \vee C) = (C^\prime \wedge (A \vee B)) \vee (C^\prime \wedge C) = V \vee (C^\prime \wedge C)
\end{equation} 
Because $V \leq C$, it must be that $V \leq C^\prime \wedge C$. Therefore, $V \vee (C^\prime \wedge C) = C^\prime \wedge C$, confirming $C^\prime \leq C$.

Using the same logic, let $A^\prime$ be the unique relative complement of $B\vee C$ in $[V, T]$. Symmetrically, $(B \vee C) \wedge A^\prime = V$, $(B \vee C) \vee A^\prime = T$, and $A^\prime \leq A$.

Looking at the intersection property of $C^\prime$, $(A\vee B) \wedge C^\prime = V$. Distributing this leaves $(A \wedge C^\prime) \vee (B \wedge C^\prime) = V$. Since both parts must be greater than or equal to $V$, it locks in:
\begin{equation}
\label{eq:eq_v}
A \wedge C^\prime = V \quad \text{and} \quad B \wedge C^\prime = V
\end{equation}
Following the pattern, $B \wedge A^\prime = V$ and $C \wedge A^\prime = V$. Also, because $A^\prime \leq A$, we know $A^\prime \wedge C^\prime \leq A \wedge C^\prime = V$. With $V \leq A^\prime$ and $V \leq C^\prime$, it is guaranteed that $A^\prime \wedge C^\prime = V$.

Now, since $A \wedge B \leq A \vee B$, we apply the multiplicative diminishing returns from Proposition \ref{prop:submodular}(ii) using $Z = C^\prime$:
\begin{equation}
\label{eq:up_ab}
\frac{f(A \wedge B)}{f(A \vee B)} \leq \frac{f((A \wedge B) \vee C^\prime)}{f((A \vee B) \vee C^\prime)} = \frac{f((A \wedge B) \vee C^\prime)}{f(T)}
\end{equation}

Similarly, since $B \wedge C \leq B \vee C$, plugging in $Z = A^\prime$ yields:
\begin{equation}
\label{eq:up_bc}
\frac{f(B \wedge C)}{f(B \vee C)} \leq \frac{f((B \wedge C) \vee A^\prime)}{f((B \vee C) \vee A^\prime)} = \frac{f((B \wedge C) \vee A^\prime)}{f(T)}
\end{equation}

Combining (\ref{eq:up_ab}) and (\ref{eq:up_bc}) creates a solid upper bound for the right side of our target inequality:
\begin{equation}
\label{eq:add_up_ab_bc}
\frac{f(A \wedge B)}{f(A \vee B)} + \frac{f(B \wedge C)}{f(B \vee C)} \leq \frac{f((A \wedge B) \vee C^\prime) + f((B \wedge C) \vee A^\prime)}{f(T)}
\end{equation}

Concurrently, by $f$'s monotonicity, we know $f(A \vee C) \leq f(T)$. This creates a lower bound for the left side of our target inequality:
\begin{equation}
\label{eq:lhs_bound}
1 + \frac{f(A \wedge C)}{f(A \vee C)} \geq 1 + \frac{f(A \wedge C)}{f(T)} = \frac{f(T) + f(A \wedge C)}{f(T)}
\end{equation}

Comparing (\ref{eq:add_up_ab_bc}) and (\ref{eq:lhs_bound}), we just need to ensure the numerators hold up:
\begin{equation}
\label{eq:to_prove_triang}
f((A \wedge B) \vee C^\prime) + f((B \wedge C) \vee A^\prime) \leq f(T) + f(A \wedge C)
\end{equation}

Since $f$ is supermodular, we can bound the left side of (\ref{eq:to_prove_triang}) using the join and meet of its inputs:
\begin{equation}
\label{eq:to_prove_triang_2}
f((A \wedge B)\vee C^\prime ) + f((B \wedge C) \vee A^\prime ) \leq f(X_{join}) + f(X_{meet})
\end{equation}
where $X_{join} = ((A\wedge B)\vee C^\prime) \vee ((B\wedge C)\vee A^\prime)$ and $X_{meet} = ((A\wedge B)\vee C^\prime) \wedge ((B\wedge C)\vee A^\prime)$.

For $X_{join}$, since $A^\prime \leq A$ and $C^\prime \leq C$, we see $X_{join} \leq (A \vee C) \vee (C \vee A) = A \vee C$. By monotonicity:
\begin{equation}
\label{eq:to_prove_triang_3p}
f(X_{join}) \leq f(A \vee C) \leq f(T)
\end{equation}

For $X_{meet}$, we leverage the distributive lattice to fully expand the intersection:
\begin{equation}
\label{eq:to_prove_triang_3}
X_{meet} = (A \wedge B \wedge C) \vee (A \wedge B \wedge A^\prime) \vee (C^\prime \wedge B \wedge C) \vee (C^\prime \wedge A^\prime)
\end{equation}
Applying the identities from (\ref{eq:eq_v}) streamlines this down to:
\begin{equation}
X_{meet} = V \vee (A \wedge V) \vee (V \wedge C) \vee V = V
\end{equation}
Given $V = A \wedge B \wedge C$, it is obvious that $V \leq A \wedge C$. Once again, by monotonicity:
\begin{equation}
\label{eq:to_prove_triang_4}
f(X_{meet}) = f(V) \leq f(A \wedge C)
\end{equation}

Adding the bounds from (\ref{eq:to_prove_triang_3p}) and (\ref{eq:to_prove_triang_4}) totals exactly $f(T) + f(A \wedge C)$, confirming (\ref{eq:to_prove_triang}) and satisfying the triangle inequality.
\end{proof}

Keep in mind that a relatively complemented distributive lattice is not automatically a Boolean algebra, as it does not demand global bounds (like a universal top or bottom element). Familiar examples include the lattice of finite subsets of natural numbers (no top element) and co-finite subsets of natural numbers (no bottom element). Theorem \ref{th:super-log-sub} works because it restricts its focus entirely to the closed interval sublattice $[A \wedge B \wedge C, A \vee B \vee C]$, which operates locally as a bounded Boolean algebra.

Furthermore, Theorem \ref{th:super-log-sub} covers much more ground than previous results concerning the generalized Consonni and Todeschini similarity \cite{Badica2025}. Those older proofs were restricted to Boolean algebras, required a specific lower bound $\bot$ where $f(\bot)=0$, and demanded $f$ be differentiable, typically built as a convex function combined with a measure. Theorem \ref{th:super-log-sub} clears away these constraints, rendering those earlier results as straightforward corollaries to this wider theorem.

\section{Result for Submodular Valuations}
\label{sec:sub}

According to Theorem 4 in \cite{Kosub2019}, a modified Jaccard distance based on the symmetric difference handles the triangle inequality perfectly when evaluating sets with a submodular valuation. In this section, we pull this result out of basic Boolean lattices (sets) and stretch it to fit sectionally complemented distributive lattices. 

A main challenge arises because the symmetric difference is not natively defined across arbitrary lattices. However, in a sectionally complemented distributive lattice, you can consistently map out the symmetric difference using the relative complement operation.

\begin{proposition}\label{prop:set-diff-consistency} (Relative set difference works consistently in sectionally complemented distributive lattices)
For any four elements $A,B,X,Y$ in a sectionally complemented distributive lattice where $A \vee B \leq X$ and $A \vee B \leq Y$:
\begin{equation}
\label{eq:set-diff-consistency}
A \wedge B^\prime_X = A \wedge B^\prime_Y
\end{equation}
where $B^\prime_X$ is the relative complement of $B$ in the interval $[\bot, X]$.
\end{proposition}

\begin{proof}
By definition, the relative complement dictates $B \wedge B^\prime_X = \bot$ and $B \vee B^\prime_X = X$. Evaluating the meet and join of $(A \wedge B^\prime_X)$ with $B$ gives:
\begin{equation}
\label{eq:set-diff-consistency-proof-1}
(A \wedge B^\prime_X) \wedge B = A \wedge (B^\prime_X \wedge B) = A \wedge \bot = \bot
\end{equation}
Applying the distributive law:
\begin{equation}
\label{eq:set-diff-consistency-proof-2}
(A \wedge B^\prime_X) \vee B = (A \vee B) \wedge (B^\prime_X \vee B) = (A \vee B) \wedge X = A \vee B
\end{equation}

These exact same equations hold true when swapping $X$ for $Y$. According to Theorem 10 in \cite{Birkhoff1973} (p. 12), complements in a distributive lattice are unequivocally unique. Because $A \wedge B^\prime_X$ and $A \wedge B^\prime_Y$ both produce $\bot$ during a meet with $B$, and $A \vee B$ during a join with $B$, they have to be the exact same unique element within the interval $[\bot, A \vee B]$. Thus, equality (\ref{eq:set-diff-consistency}) stands.
\end{proof}

Following Proposition \ref{prop:set-diff-consistency}, the evaluation of the lattice difference $A \setminus B$ is independent of the choice of the upper bound $X$, provided $X \geq A \vee B$. Consequently, both the lattice difference and the symmetric difference $A \Delta B$ can be consistently defined in any sectionally complemented distributive lattice using any such arbitrary upper bound $X$:
\begin{equation}\label{eq:set-sym-diff}
A \setminus B = A \wedge B^\prime_X \quad \text{and} \quad A \Delta B = (A \setminus B) \vee (B \setminus A)
\end{equation}

With this symmetric-difference operator well-defined, we formulate the alternative Jaccard distance for submodular valuations as a piecewise function to mathematically safeguard against division by zero when evaluating the bottom elements (since $f(\bot) = 0$):
\begin{equation}
\label{eq:sym-diff-jaccard-dist}
d_{f,\Delta}(A, B) = 
\begin{cases} 
0 & \text{if } A = B = \bot \\ 
\frac{f(A \Delta B)}{f(A \vee B)} & \text{otherwise} 
\end{cases}
\end{equation}

With this framework established, we introduce our next generalization.
\begin{theorem}\label{th:submodular-sym-diff}
Assume $L$ is a sectionally complemented distributive lattice, and $f$ is a positive, monotone, and submodular valuation function on $L$ where $f(\bot) = 0$. In this scenario, the distance $d_{f,\Delta}$ defined in (\ref{eq:sym-diff-jaccard-dist}) satisfies the triangle inequality.
\end{theorem}

\begin{proof}
For any three elements $A, B, C \in L$, if $A = B = C = \bot$, the distance between any pair is $0$ by definition (\ref{eq:sym-diff-jaccard-dist}), and the triangle inequality $0 \leq 0 + 0$ holds trivially. 

For all other cases, let $T = A \vee B \vee C > \bot$. By definition, the closed interval sublattice $[\bot, T]$ operates as a Boolean algebra. 

Because $A, B, C$ all reside inside this Boolean algebra, standard Boolean structural identities naturally apply to them. Specifically, the core inclusion property for the symmetric difference stays intact:
\begin{equation}
\label{eq:sym-diff-inclusion}
A \Delta B \leq (A \Delta C) \vee (C \Delta B)
\end{equation}

Since $f$ is submodular and $f(\bot) = 0$, it is inherently subadditive, meaning $f(X \vee Y) \leq f(X) + f(Y)$ for any $X, Y \in L$. Pairing $f$'s monotonicity with (\ref{eq:sym-diff-inclusion}) and applying this subadditivity yields:
\begin{equation}
\label{eq:sym-diff-subadd}
f(A \Delta B) \leq f((A \Delta C) \vee (C \Delta B)) \leq f(A \Delta C) + f(C \Delta B)
\end{equation}

Equation (\ref{eq:sym-diff-subadd}) proves that the numerator $f(A \Delta B)$ functions as an unnormalized metric (handling the triangle inequality entirely on its own). From here, since $f$ is monotone and subadditive, the algebra required to prove the triangle inequality for the normalized fraction $d_{f,\Delta}(A, B)$ matches the set-theoretic steps outlined in Theorem 4 of \cite{Kosub2019}.
\end{proof}

A direct result of Theorem \ref{th:submodular-sym-diff} is realizing that a universal maximum element ($\top$) isn't mathematically necessary to keep the symmetric-difference Jaccard distance functioning as a metric. By leaning only on sectional complementation, the theorem shows that the localized Boolean structure of the interval $[\bot, A \vee B \vee C]$ carries all the algebra required to maintain the subadditivity of $f(A \Delta B)$. 

Furthermore, this finding exposes a hard structural divide regarding how the Jaccard distance must be generalized based on the valuation function. For modular and supermodular valuations, the standard formulation $1 - \frac{f(A \wedge B)}{f(A \vee B)}$ comfortably handles the triangle inequality. However, if the valuation is submodular, you must pivot to the symmetric-difference formulation $\frac{f(A \Delta B)}{f(A \vee B)}$. As Section \ref{sec:nec} will clarify, this pivot isn't a stylistic choice, but it's mathematically mandatory, because supermodularity is an explicit prerequisite for the standard formulation to act as a metric. Proposition \ref{prop:set-diff-consistency} simply ensures this required pivot to symmetric difference remains viable even in open-ended distributive lattices.

\section{Necessary Conditions for the Triangle Inequality}
\label{sec:nec}

In this section, we outline the necessary conditions a valuation function $f$ must meet for the generalized Jaccard distance $d_{f,J}$, defined by (\ref{eq:gen-jaccard-dist}), to function as a valid metric on an arbitrary lattice. This creates a straightforward test: if $f$ fails any of these conditions, the standard generalized Jaccard formula cannot possibly satisfy the triangle inequality.

\begin{theorem}\label{th:super-nec-cond}
If the distance measure $d_{f,J}$ generated by a strictly positive real valuation function $f$ on an arbitrary lattice satisfies the triangle inequality, it is mandatory that $f$ and $\frac{1}{f}$ are both supermodular, and that $f$ is monotone.
\end{theorem}

\begin{proof}
By definition, the triangle inequality directly equates to inequality (\ref{eq:triangle_ineq_equiv}):
$$1+ \frac{f(A \wedge C)}{f(A \vee C)} \geq \frac{f(A \wedge B)}{f(A \vee B)} + \frac{f(B \wedge C)}{f(B \vee C)} \quad \forall A,B,C \in L$$

To begin, substitute $B = A \vee C$. Applying the absorption law, $A \wedge (A \vee C) = A$ and $C \wedge (A \vee C) = C$. The inequality trims down to:
\begin{equation}
\label{eq:triangle_ineq_equiv_1}
1+ \frac{f(A \wedge C)}{f(A \vee C)} \geq \frac{f(A)}{f(A \vee C)} + \frac{f(C)}{f(A \vee C)}
\end{equation}
Multiplying both sides by the strictly positive $f(A \vee C)$ lands exactly on the definition of supermodularity (\ref{eq:supermodular}). Therefore, $f$ must be supermodular.

Next, set $B = A \wedge C$. By the absorption law, $A \vee (A \wedge C) = A$ and $C \vee (A \wedge C) = C$. The inequality then shrinks to:
\begin{equation}
\label{eq:triangle_ineq_equiv_2}
1+ \frac{f(A \wedge C)}{f(A \vee C)} \geq \frac{f(A \wedge C)}{f(A)} + \frac{f(A \wedge C)}{f(C)}
\end{equation}
Dividing both sides by the strictly positive $f(A \wedge C)$ yields the supermodularity condition for the reciprocal function $\frac{1}{f}$. Consequently, $\frac{1}{f}$ must also be supermodular.

Finally, to verify monotonicity, take two elements $A$ and $B$ where $A \leq B$. Substitute $C = A$ into the target inequality. Since $A \leq B$, we know $A \wedge B = A$ and $A \vee B = B$. The inequality reduces to:
\begin{equation}
\label{eq:triangle_ineq_equiv_3}
1+ \frac{f(A)}{f(A)} \geq \frac{f(A)}{f(B)} + \frac{f(A)}{f(B)} \implies 2 \geq \frac{2f(A)}{f(B)}
\end{equation}
Multiplying by the strictly positive $f(B)$ and dividing by 2 leaves $f(B) \geq f(A)$, satisfying the definition of monotonicity (\ref{eq:monotone}).
\end{proof}

It is helpful to point out that demanding both $f$ and $\frac{1}{f}$ be supermodular is actually a softer mathematical constraint than demanding $f$ be both supermodular and $\log$-submodular. Specifically, it is easy to prove that:
\begin{itemize}
    \item[i)] If $f$ is supermodular and $\log$-submodular, then $\frac{1}{f}$ is definitely supermodular.
    \item[ii)] Conversely, if $\frac{1}{f}$ is supermodular and $\log$-submodular, then $f$ is supermodular.
\end{itemize}

Theorem \ref{th:super-nec-cond} provides a quick filtering tool for new distance measures. If a proposed real, positive valuation function $f$ isn't monotone, isn't supermodular, or if its reciprocal $\frac{1}{f}$ fails supermodularity, you can immediately rule out $d_{f,J}$ as a true metric.

Stepping back, the results from Theorem \ref{th:super-log-sub} and Theorem \ref{th:super-nec-cond} highlight a deep theoretical link between a valuation function's supermodularity and the generalized Jaccard distance's metricity, firmly marking the structural boundaries for applying this metric to arbitrary data spaces. 

\section{Practical Implications and Application Scenarios}
\label{sec:applications}

While Sections \ref{sec:modular} through \ref{sec:nec} focus on structural proofs, they directly resolve strict algebraic limitations in applied computer science and information theory. By eliminating the requirement for Boolean distributivity and global lattice bounds, these results allow the generalized Jaccard distance to be computed safely on non-classical data structures. Below, we map these theoretical relaxations to specific computational domains and active research problems.

\subsection{Modular Valuations in Quantum Information Theory}

Classical information processing relies on Boolean algebra, which strictly requires a distributive lattice. Quantum mechanics, however, inherently violates classical distributive logic. As established by Birkhoff and von Neumann \cite{Birkhoff1936}, the logical propositions of a quantum system---modeled as closed subspaces of a Hilbert space---form an orthomodular lattice, which is strictly non-distributive. 

Because previous metricity proofs for the generalized Jaccard distance \cite{Kosub2019} explicitly utilized distributivity, they cannot be applied to quantum logic. Theorem \ref{th:mod} resolves this limitation. In quantum information theory, a fundamental task is quantifying the distinguishability between two quantum states. Defining metric spaces over quantum effect algebras and orthomodular structures remains a highly active area of research \cite{Mishra2025}. If we define the valuation function $f$ as the trace of a quantum projection operator (or as a quantum probability measure via Gleason's Theorem \cite{Gleason1957}), $f$ operates as a strictly positive, monotone, and modular measure. By proving that distributivity is mathematically unnecessary under modular valuations, Theorem \ref{th:mod} guarantees that the standard generalized Jaccard distance ($1 - \frac{\text{Tr}(P \wedge Q)}{\text{Tr}(P \vee Q)}$) provides a rigorously valid metric for comparing quantum events within their native, non-distributive algebraic structure.

\subsection{Supermodular Valuations in Conceptual Hierarchies and Fuzzy Measures}

Our findings also extend the generalized Jaccard metric to supermodular and $\log$-submodular valuation functions on relatively complemented distributive lattices. The primary computational advantage of this relaxation is the complete removal of the need for a universal top ($\top$) or bottom ($\bot$) element.

This result directly optimizes distance computations in Formal Concept Analysis \cite{Ganter2024} and systems utilizing fuzzy sets \cite{Zadeh1965}. In standard applications, intersection (minimum) and union (maximum) operations form distributive lattices. However, in distributed databases, a global top element is typically undefined; computing a universal maximum across an open-ended universe is algorithmically intractable and unnecessary for localized queries. 

Because relatively complemented lattices ensure that any closed interval $[U, V]$ operates locally as a Boolean algebra, Theorem \ref{th:super-log-sub} allows clustering algorithms to execute safely within defined local boundaries. Furthermore, in fuzzy measure theory, supermodular valuation functions are standard tools used to model synergistic interactions, where the combined evaluation of two features strictly exceeds the sum of their independent evaluations \cite{Grabisch1996}. Theorem \ref{th:super-log-sub} ensures that similarity algorithms utilizing these synergistic measures remain geometrically consistent without requiring global data bounds. Simultaneously, the required $\log$-submodularity constraint mathematically dampens this growth, ensuring that while the absolute valuation exhibits supermodular scaling, the relative proportions remain strictly bounded (Proposition \ref{prop:submodular}). This prevents the normalized distance fractions from distorting the metric space.

\subsection{Submodular Valuations in Temporal Databases and Data Streams}

For submodular valuation functions, Theorem \ref{th:submodular-sym-diff} establishes that the symmetric-difference Jaccard distance satisfies the triangle inequality on sectionally complemented distributive lattices. This directly solves a geometric stability problem for machine learning algorithms processing continuous, dynamic data.

In algorithm design, submodularity is the standard mathematical framework for modeling diminishing returns. It continuously drives optimization models for feature selection, sensor placement, and influence maximization \cite{Kempe2003, Krause2008}, and it is actively utilized in high-volume workloads to maximize diversity using the Jaccard distance \cite{MaxSumDiv2026}.

This valuation profile naturally characterizes temporal databases and continuous Internet of Things (IoT) data streams \cite{Snodgrass1999}. Because time advances continuously, the underlying data structure possesses no fixed global upper bound (no $\top$ element), though it originates from a defined initialization state ($\bot$). Sectionally complemented distributive lattices accurately map this topology. In stream mining, algorithms typically evaluate data using sliding windows or localized historical sequences \cite{Aggarwal2007}. Establishing this time window defines a temporary upper bound $X$, which guarantees the interval $[\bot, X]$ remains structurally consistent. Furthermore, the symmetric difference $A \Delta B$ is the precise mathematical operator required to quantify the transition (the delta) between two sliding windows. By proving that the symmetric-difference metric is robust within these local sectional intervals, Theorem \ref{th:submodular-sym-diff} ensures that distance-based machine learning algorithms can process infinite streams indefinitely without violating the triangle inequality.

\section{Discussion of Open Theoretical Problems}\label{sec:open-problems}

While the generalizations detailed in this paper expand how and where the Jaccard distance can be used, they also define hard structural limits. Going forward, there are four key open mathematical problems left to tackle:

\begin{itemize}
    \item \textbf{Metricity in non-modular lattices:} We proved that a modular valuation yields a valid metric on an arbitrary lattice without leaning on lattice modularity axioms. However, Rutherford \cite{Rutherford1965} notes that any lattice supporting a monotone and modular valuation is categorized as a metric lattice, and all metric lattices are modular. This means a modular valuation simply cannot exist on a non-modular lattice (like the pentagon lattice $N_5$). Building a generalized Jaccard metric for non-modular structures will therefore require supermodular or submodular valuations, assuming the distributivity requirements of those valuations can be successfully relaxed in future studies.
    
    \item \textbf{Arbitrary lattices for supermodular valuations:} While we established metricity for supermodular and $\log$-submodular valuations, the proof relies on the lattice being both relatively complemented and distributive. It is still unknown whether the standard generalized Jaccard formula holds up to the triangle inequality for these valuations on arbitrary lattices that lack those two specific traits.
    
    \item \textbf{Arbitrary lattices for submodular valuations:} Similarly, we proved metricity for the symmetric-difference Jaccard formula under submodular valuations by leveraging sectionally complemented distributive lattices. Whether this symmetric-difference metric remains valid on arbitrary lattices missing these structural properties is still an open question.
    
    \item \textbf{The sufficiency of necessary conditions:} We proved that if the standard generalized Jaccard formula is a metric, $f$ and $\frac{1}{f}$ must absolutely be supermodular. The reverse of this statement remains untested: if $f$ and $\frac{1}{f}$ are supermodular, is that enough on its own to guarantee the triangle inequality? Also, it is unclear if the stricter $\log$-submodularity constraint used in Theorem \ref{th:super-log-sub} is actually required, or if it can be loosened to match the more basic necessary conditions established later in the paper.
\end{itemize}

\section{Conclusions}
\label{sec:concl}

This paper generalized the Jaccard distance to evaluate similarities between elements across arbitrary lattices. By carefully stripping away the Boolean and distributive axioms that older literature relied on, we laid out new sufficient and necessary conditions for the generalized Jaccard distance to hold up to the triangle inequality. 

We showed that distributivity is mathematically unnecessary when dealing with a modular valuation function, proving that strictly positive, monotone, and modular valuations yield a valid metric on completely arbitrary lattices. For supermodular and $\log$-submodular valuations, we verified the triangle inequality holds in relatively complemented distributive lattices, removing the long-standing need for global lattice bounds. We also adapted the symmetric-difference Jaccard formula for submodular valuations, confirming its metricity in sectionally complemented distributive lattices. Finally, we proved that supermodularity and monotonicity are explicit, baseline requirements for the standard generalized Jaccard formula to function as a metric at all.

Beyond the math, relaxing these structural axioms offers immediate practical value. Stripping away the distributivity requirement allows this generalized metric to work seamlessly on the non-distributive orthomodular lattices common in quantum information theory. Similarly, dropping the need for global lattice bounds lets these distance metrics run natively on fuzzy sets, formal concept hierarchies, and infinite temporal or IoT data streams.

Moving forward, research will likely split in two useful directions. Theoretically, it is important to figure out if the distributivity requirements for supermodular and submodular valuations can be completely removed, whether the necessary and sufficient conditions can be perfectly unified, and how generalized Jaccard metrics might be adjusted for strictly non-modular structures. On the applied side, future work should focus on plugging these generalized distance metrics into machine learning algorithms, like $k$-medoids and conceptual clustering, to benchmark their practical performance and computational speed on non-classical data structures and endless data streams.

\bibliographystyle{alpha} 
\bibliography{sample}

\end{document}